\documentclass[11pt]{article}

\usepackage[utf8]{inputenc}
\usepackage{amsmath}
\usepackage{amssymb}
\usepackage{amsthm}
\usepackage[english]{babel}
\usepackage{fullpage}
\usepackage{graphics}
\usepackage[colorlinks=true,citecolor=blue,linkcolor=blue]{hyperref}
\usepackage{natbib}
\usepackage{todonotes}

\addto\extrasenglish{%
}

\newtheorem{proposition}{Proposition}
\newtheorem{lemma}[proposition]{Lemma}
\newtheorem{corollary}[proposition]{Corollary}
\newtheorem{theorem}[proposition]{Theorem}

\newcommand{\R}{\mathbb{R}}

\renewcommand{\P}{\mathcal{P}}
\newcommand{\U}{\mathcal{U}}
\newcommand{\X}{\mathcal{X}}
\newcommand{\indicator}[1]{\mathbf{1}\left[{#1}\right]}

\newcommand{\KL}[2]{D\left(#1 \middle\| #2 \right)}
\newcommand{\norm}[1]{\left\| #1 \right\|}

\DeclareMathOperator*{\argmin}{argmin}
\DeclareMathOperator{\err}{err}
\DeclareMathOperator{\Exp}{\mathbf{E}}
\DeclareMathOperator{\VC}{VC}

\begin{document}

\title{The information-theoretic value of unlabeled data in semi-supervised learning}
\author{Alexander Golovnev\thanks{Harvard University, Cambridge, MA, USA. Supported by a Rabin Postdoctoral Fellowship.} \and D\'avid P\'al\thanks{Yahoo Research, New York, NY, USA} \and Bal\'azs Sz\"or\'enyi\footnotemark[2]}

\maketitle

\begin{abstract}
We quantify the separation between the numbers of labeled examples required to
learn in two settings: Settings \emph{with} and \emph{without} the knowledge of
the distribution of the unlabeled data. More specifically, we prove a separation
by $\Theta(\log n)$ multiplicative factor for the class of projections over
the Boolean hypercube of dimension $n$. We prove that there is no separation
for the class of all functions on domain of any size.

Learning with the knowledge of the distribution (a.k.a. \emph{fixed-distribution
learning}) can be viewed as an idealized scenario of semi-supervised learning
where the number of unlabeled data points is so great that the unlabeled
distribution is known exactly. For this reason, we call the separation the
\emph{value of unlabeled data}.
\end{abstract}

\section{Introduction}
\label{section:introduction}

\citet{Hanneke-2016} showed that for any class $C$ of Vapnik-Chervonenkis
dimension $d$ there exists an algorithm that $\epsilon$-learns any target
function from $C$ under any distribution from $O\left(\frac{d +
\log(1/\delta)}{\epsilon}\right)$ labeled examples with probability at least
$1-\delta$. For this paper, it is important to stress that Hanneke's algorithm
does \emph{not} receive the distribution of unlabeled data as input. On the
other hand, \citet{Benedek-Itai-1991} showed that for any class $C$ and any
distribution there exists an algorithm that $\epsilon$-learns any target from
$C$ from $O \left( \frac{\log N_{\epsilon/2} + \log
(1/\delta)}{\epsilon}\right)$ labeled examples with probability at least
$1-\delta$ where $N_{\epsilon/2}$ is the size of an $\frac{\epsilon}{2}$-cover
of $C$ with respect to the disagreement metric $d(f,g) = \Pr[f(x) \neq g(x)]$.
Here, it is important to note that Benedek and Itai construct for each
distribution a separate algorithm. In other words, they construct a family of
algorithms indexed by the (uncountably many) distributions over the domain.
Alternatively, we can think of Benedek-Itai's family of algorithms as a single
algorithm that receives the distribution as an input. It is known that
$N_\epsilon = O(1/\epsilon)^{O(d)}$; see \citet{Dudley-1978}. Thus, ignoring
$\log(1/\epsilon)$ factor, Benedek-Itai bound is never worse than Hanneke's
bound.

As we already mentioned, Benedek-Itai's algorithm receives as input the
distribution of unlabeled data. The algorithm uses it to construct an
$\frac{\epsilon}{2}$-cover. Unsurprisingly, there exist distributions which have
a small $\frac{\epsilon}{2}$-cover and thus sample complexity of Benedek-Itai's
algorithm on such distributions is significantly lower then the Hanneke's bound.
For instance, a distribution concentrated on a single point has an
$\frac{\epsilon}{2}$-cover of size $2$ for any positive $\epsilon$.

However, an algorithm does not need to receive the unlabeled distribution in
order to enjoy low sample complexity. For example, empirical risk minimization
(ERM) algorithm needs significantly less labeled examples to learn any target
under some unlabeled distributions. For instance, if the distribution is
concentrated on a single point, ERM needs only one labeled example to learn any
target. One could be lead to believe that there exists an algorithm that does
\emph{not} receive the unlabeled distribution as input and achieves Benedek-Itai
bound (or a slightly worse bound) for \emph{every} distribution. In fact, one
could think that ERM or Hanneke's algorithm could be such algorithms. If ERM,
Hanneke's algorithm, or some other distribution-independent algorithm had sample
complexity that matches (or nearly matches) the optimal distribution-specific
sample complexity for \emph{every} distribution, we could conclude that the
knowledge of unlabeled data distribution is completely useless.

As \citet{Darnstadt-Simon-Szorenyi-2013} showed this is not the case. They showed
that \emph{any} algorithm for learning projections over $\{0,1\}^n$ that does
not receive the unlabeled distribution as input, requires, for some data
unlabeled distributions, more labeled examples than the Benedek-Itai bound.
However, they did not quantify this gap beside stating that it grows without
bound as $n$ goes to infinity.

In this paper, we quantify the gap by showing that \emph{any}
distribution-independent algorithm for learning the class of projections over
$\{0,1\}^n$ requires, for some unlabeled distributions, $\Omega(\log n)$ times
as many labeled examples as Benedek-Itai bound.
\citet{Darnstadt-Simon-Szorenyi-2013} showed the gap for any class with
Vapnik-Chervonenkis dimension $d$ is at most $O(d)$. It is well known that
Vapnik-Chervonenkis dimensions of projections over $\{0,1\}^n$ is $\Theta(\log
n)$. Thus our lower bound matches the upper bound $O(d)$. To better understand
the relationship of the upper and lower bounds, we illustrate the situation for
the class of projections over $\{0,1\}^n$ in
Figure~\ref{figure:sample-complexity}.

In contrast, we show that for the class of \emph{all} functions (on any domain)
there is no gap between the two settings. In other words, for learning a target
from the class of all functions, unlabeled data are in fact useless. This
illustrates the point that the gap depends in a non-trivial way on the
combinatorial structure of the function class rather than just on the
Vapnik-Chervonenkis dimension.

\begin{figure}
\centering
\includegraphics[scale=1.0]{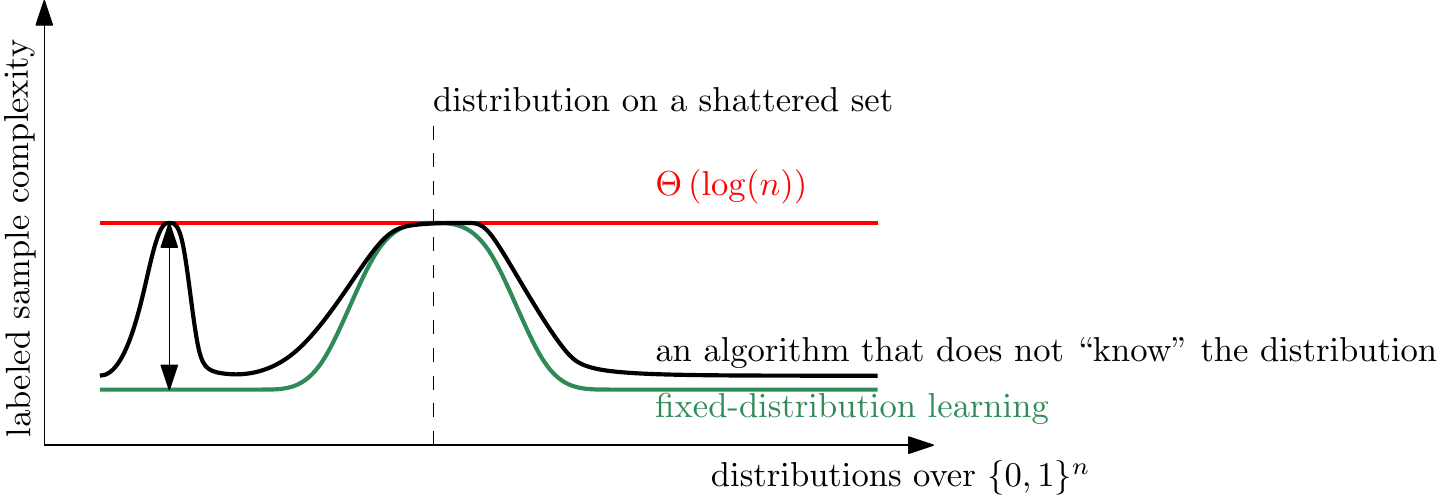}
\caption{\small The graph shows sample complexity bounds of learning a class of
projections over the domain $\{0,1\}^n$ under various unlabeled distributions.
We assume that $\epsilon$ and $\delta$ are constant, say, $\epsilon = \delta =
\frac{1}{100}$. The graph shows three lines. The red horizontal line is
Hanneke's bound for the class of projections, which is $\Theta(\VC(C_n)) =
\Theta(\log n)$. The green line is the Benedek-Itai bound. The green line
touches the red line for certain distributions, but is lower for other
distributions. In particular, for certain distributions the green line is
$O(1)$. The dashed line corresponds to a particular distribution on a shattered
set. This is where the green line and red line touch. Furthermore, here the
upper bound coincides with the lower bound for that particular distribution. The
black line is the sample complexity of an arbitrary
\emph{distribution-independent} algorithm. For example, the reader can think of
the ERM or Hanneke's algorithm. We prove that there exist a distribution where
the black line is $\Omega(\log n)$ times higher than the green line. This
separation is indicated by the double arrow.}
\label{figure:sample-complexity}
\end{figure}

The paper is organized as follows. In \autoref{section:related-work} we review
prior work. \autoref{section:preliminaries} gives the necessary definitions and
basic probabilistic tools. In \autoref{section:projections} we give the proof of
the separation result for projections. In \autoref{section:all-functions} we
prove that there is no gap for the class of all functions. For completeness, in \autoref{section:epsilon-cover} we give a proof of a simple upper bound
$O(1/\epsilon)^{O(d)}$ on the size of the minimum $\epsilon$-cover, and in \autoref{section:fixed-distribution-learning} we give a proof of
Benedek-Itai's $O \left( \frac{\log N_{\epsilon/2} + \log
(1/\delta)}{\epsilon}\right)$ sample complexity upper bound.

\section{Related work}
\label{section:related-work}

The question of whether knowledge of unlabeled data distribution helps was
proposed and initially studied by \citet{Ben-David-Lu-Pal-2008}; see also
\citet{Lu-2009}. However, they considered only classes with Vapnik-Chervonenkis
dimension at most $1$, or classes with Vapnik-Chervonenkis dimension $d$ but
only distributions for which the size of the $\epsilon$-cover is
$\Theta(1/\epsilon)^{\Theta(d)}$, i.e. the $\epsilon$-cover is as large as it
can be.\footnote{For any concept class with Vapnik-Chervonenkis dimension $d$
and any distribution, the size of the smallest $\epsilon$-cover is at most
$O(1/\epsilon)^{O(d)}$.} In these settings, for constant $\epsilon$ and $\delta$,
the separation of labeled sample complexities is at most a constant factor,
which is exactly what \citet{Ben-David-Lu-Pal-2008} proved. In these settings,
unlabeled data are indeed useless. However, these results say nothing about
distributions with $\epsilon$-cover of small size and it ignores the dependency
on the Vapnik-Chervonenkis dimension.

The question was studied in earnest by \citet{Darnstadt-Simon-Szorenyi-2013} who
showed two major results. First, they show that for any non-trivial concept
class $C$ and for every distribution, the ratio of the labeled sample
complexities between distribution-independent and distribution-dependent
algorithms is bounded by the Vapnik-Chervonenkis dimension. Second, they show
that for the class of projections over $\{0,1\}^n$, there are distributions
where the ratio grows to infinity as a function of $n$.

In learning theory, the disagreement metric and $\epsilon$-cover were introduced
by \citet{Benedek-Itai-1991} but the ideas are much older; see
e.g.~\citet{Dudley-1978, Dudley-1984}. The $O(1/\epsilon)^{O(d)}$ upper bound on
size of the smallest $\epsilon$-cover is by \citet[Lemma 7.13]{Dudley-1978}; see
also \citet[Chapter 4]{Devroye-Lugosi-2000} and \citet{Haussler-1995}. We
present a proof of $O(1/\epsilon)^{O(d)}$ upper bound in
\autoref{section:epsilon-cover}.

For any distribution-independent algorithm and any class $C$ of
Vapnik-Chervonenkis dimension $d \ge 2$ and any $\epsilon \in (0,1)$ and $\delta
\in (0,1)$, there exists a distribution over the domain and a concept which
requires at least $\Omega \left(\frac{d + \log(1/\delta)}{\epsilon}\right)$
labeled examples to $\epsilon$-learn with probability at least $1 - \delta$;
see~\citet[Theorem 5.3]{Anthony-Bartlett-1999} and
\citet{Blumer-Ehrenfeucht-Haussler-Warmuth-1989,
Ehrenfeucht-Haussler-Kearns-Valiant-1989}. The proof of the lower bound
constructs a distribution that does \emph{not} depend on the algorithm. The
distribution is a particular distribution over a fixed set shattered by $C$. So
even an algorithm that knows the distribution requires $\Omega \left(\frac{d +
\log(1/\delta)}{\epsilon}\right)$ labeled examples.

\section{Preliminaries}
\label{section:preliminaries}

Let $\X$ be a non-empty set. We denote by $\{0,1\}^\X$ the class of all
functions from $\X$ to $\{0,1\}$. A \emph{concept class over a domain $\X$} is a
subset $C \subseteq \{0,1\}^\X$. A \emph{labeled example} is a pair $(x,y) \in
\X \times \{0,1\}$.

A \emph{distribution-independent learning algorithm} is a function
$A:\bigcup_{m=0}^\infty \left(\X \times \{0,1\} \right)^m \to \{0,1\}^\X$. In
other words, the algorithm gets as input a sequence of labeled examples $(x_1,
y_1), (x_2, y_2), \dots, (x_m, y_m)$ and outputs a function from $\X$ to
$\{0,1\}$. We allow the algorithm to output a function that does not belong to
$C$, i.e., the algorithm can be improper. A \emph{distribution-dependent
algorithm} is a function that maps any probability distribution over $\X$ to a
distribution-independent algorithm.

Let $P$ be a probability distribution over a domain $\X$. For any two functions
$f:\X \to \{0,1\}$, $g:\X \to \{0,1\}$ we define the disagreement pseudo-metric
$$
d_P(f,g) = \Pr_{X \sim P}[f(X) \neq g(X)] \; .
$$
Let $C$ be a concept class over $\X$, let $c \in C$, let $\epsilon, \delta \in (0,1)$.
Let  $X_1, X_2, \dots, X_m$ be an i.i.d. sample from $P$. We define the corresponding
labeled sample $T = ((X_1, c(X_1)), (X_2, c(X_2)), \dots, (X_m, c(X_m)))$.
We say that an algorithm $A$, \emph{$\epsilon$-learns} target $c$ from $m$ samples
with probability at least $1 - \delta$ if
$$
\Pr \left[d_P(c,A(T)) \le \epsilon \right]  \ge 1 - \delta \; .
$$
The smallest non-negative integer $m$ such that for any target $c \in C$,
the algorithm $A$, $\epsilon$-learns the target $c$ from $m$ samples
with probability at least $1-\delta$ is denoted by $m(A,C,P,\epsilon,\delta)$.

We recall the standard definitions from learning theory. For any concept $c:\X
\to \{0,1\}$ and any $S \subseteq \X$ we define $\pi(c,S) = \{x \in S ~:~ c(x) =
1 \}$. In other words, $\pi(c,S)$ is the set of examples in $S$ which $c$ labels
$1$. A set $S \subseteq \X$ is \emph{shattered} by a concept class $C$ if for
any subset $S' \subseteq S$ there exists a classifier $c \in C$ such that
$\pi(c,S) = S'$. \emph{Vapnik-Chervonenkis dimension} of a concept class $C$ is
the size of the largest set $S \subseteq \X$ shattered by $C$. A subset $C'$ of
a concept class $C$ is an \emph{$\epsilon$-cover} of $C$ for a probability
distribution $P$ if for any $c \in C$ there exists $c' \in C'$ such that
$d_P(c,c') \le \epsilon$.

To prove our lower bounds we need three general probabilistic results. The first
one is the standard Hoeffding bound. The other two are simple and intuitive
propositions. The first proposition says that if average error $d_P(c,A(T))$ is
high, the algorithm fails to $\epsilon$-learn with high probability. The second
proposition says that the best algorithm for predicting a bit based on some side
information, is to compute conditional expectation of the bit and thresholds it
at $1/2$.

\begin{theorem}[Hoeffding bound]
Let $X_1, X_2, \dots, X_n$ be i.i.d. random variables that lie in interval
$[a,b]$ with probability one and let $p=\frac{1}{n}\sum_{i=1}^n \Exp[X_i]$.
Then, for any $t \ge 0$,
\begin{align*}
\Pr \left[{\frac {1}{n}} \sum_{i=1}^n X_i \ge p + t \right] \le e^{ - 2n t^2/(a-b)^2} \; , \\
\Pr \left[{\frac {1}{n}} \sum_{i=1}^n X_i \le p - t \right] \le e^{ - 2n t^2/(a-b)^2}  \; .
\end{align*}
\end{theorem}

\begin{proposition}[Error probability vs. Expected error]
\label{proposition:error-probability-vs-expected-error}
Let $Z$ be a random variable such that $Z \le 1$ with probability one.
Then,
$$
\Pr[Z > t] \ge \frac{\Exp[Z] - t}{1 - t} \qquad \text{for any $t \in [0, 1)$.}
$$
\end{proposition}

\begin{proof}
We have
\begin{align*}
\Exp[Z]
\le t \cdot \Pr[Z \le t] + 1 \cdot \Pr[Z > t]
= t \cdot (1 - \Pr[Z > t]) + \Pr[Z > t] \; .
\end{align*}
Solving for $\Pr[Z > t]$ finishes the proof.
\end{proof}

\begin{proposition}[Predicting Single Bit]
\label{proposition:single-bit}
Let $\U$ be a finite non-empty set. Let $U,V$ be random variables (possibly
correlated) such that $U \in \U$ and $V \in \{0,1\}$ with probability one. Let
$f:\U \to \{0,1\}$ be a predictor. Then,
$$
\Pr\left[ f(U) \neq V \right]
\ge \sum_{u \in \U} \left( \frac{1}{2} - \left| \frac{1}{2} -  \Exp \left[V \, \middle| \, U = u\right] \right| \right) \cdot \Pr[U = u] \; .
$$
\end{proposition}

\begin{proof}
We have
$$
\Pr \left[ f(U) \neq V \right] = \sum_{u \in \U} \Pr \left[ f(U) \neq V \, \middle| \, U = u \right] \cdot \Pr[U = u] \; .
$$
It remains to show that
$$
\Pr\left[ f(U) \neq V \, \middle| \, U = u \right]
\ge
\frac{1}{2} - \left| \frac{1}{2} -  \Exp \left[V \, \middle| \, U = u \right] \right| \; .
$$
Since if  $U=u$, the value $f(U) = f(u)$ is fixed, and hence
\begin{align*}
\Pr\left[ f(U) \neq V \, \middle| \, U = u \right]
& \ge \min\left\{ \Pr \left[ V = 1 \, \middle| \, U = u \right], \ \Pr \left[ V = 0 \, \middle| \, U = u \right] \right\} \\
& = \min\left\{ \Exp \left[ V  \, \middle| \, U = u \right], \ 1 - \Exp \left[ V \, \middle| \, U = u \right] \right\} \\
& = \frac{1}{2} - \left| \frac{1}{2} -  \Exp \left[ V  \, \middle| \, U = u \right] \right|
\end{align*}
We used the fact that $\min\{x, 1 - x\} = \frac{1}{2} - \left| \frac{1}{2} - x \right|$ for all $x \in \R$
which can be easily verified by considering two cases: $x \ge \frac{1}{2}$ and $x < \frac{1}{2}$.
\end{proof}

\section{Projections}
\label{section:projections}

In this section, we denote by $C_n$ the class of \emph{projections} over the
domain $\X = \{0,1\}^n$. The class $C_n$ consists of $n$ functions $c_1, c_2,
\dots, c_n$ from $\{0,1\}^n$ to $\{0,1\}$. For any $i \in \{1,2,\dots,n\}$, for
any $x \in \{0,1\}^n$, the function $c_i$ is defined as $c_i((x[1], x[2], \dots,
x[n])) = x[i]$.

For any $\epsilon \in (0,\frac{1}{2})$ and $n \ge 2$, we consider a family
$\P_{n,\epsilon}$ consisting of $n$ probability distributions $P_1, P_2, \dots,
P_n$ over the Boolean hypercube $\{0,1\}^n$. In order to describe the
distribution $P_i$, for some $i$, consider a random vector $X = (X[1], X[2],
\dots, X[n])$ drawn from $P_i$. The distribution $P_i$ is a product
distribution, i.e., $\Pr[X = x] = \prod_{j=1}^n \Pr[X[j] = x[j]]$ for any $x \in
\{0,1\}^n$. The marginal distributions of the coordinates are
$$
\Pr[X[j] = 1] =
\begin{cases}
\frac{1}{2} & \text{if $j = i$,} \\
\epsilon & \text{if $j\neq i$,} \\
\end{cases}
\qquad \text{for $j=1,2,\dots,n$.}
$$
The reader should think of $\epsilon$ as a constant that does not depend on $n$,
say, $\epsilon=\frac{1}{100}$.

The following result is folklore. We include its proof for completeness. 

\begin{proposition}
\label{proposition:vc-dimension-projections}
Vapnik-Chervonenkis dimension of $C_n$ is $\lfloor \log_2 n \rfloor$.
\end{proposition}

\begin{proof}
Let us denote the Vapnik-Chervonenkis dimension by $d$. Recall that $d$ is the
size of the largest shattered set. Let $S$ be any shattered set of size $d$.
Then, there must be at least $2^d$ distinct functions in $C_n$. Hence, $d \le
\log_2 |C_n| = \log_2 n$. Since $d$ is an integer, we conclude that $d \le
\lfloor \log_2 n \rfloor$.

On the other hand, we construct a shattered set of size $\lfloor \log_2 n
\rfloor$. The set will consists of points $x_1, x_2, \dots, x_{\lfloor \log_2 n
\rfloor} \in \{0,1\}^n$. For any $i \in \{1,2,\dots,\lfloor \log_2 n \rfloor\}$
and any $j \in \{0,1,2,\dots,n-1\}$, we define $x_i[j]$ to be the $i$-th bit
in the binary representation of the number $j$. (The bit at position $i=1$ is the
least significant bit.) It is not hard to see that for any $v \in
\{0,1\}^{\lfloor \log_2 n \rfloor}$, there exists $c \in C_n$ such that $v =
(c(x_1), c(x_2), \dots, c(x_{\lfloor \log_2 n \rfloor}))$. Indeed, given $v$,
let $k \in \{0,1,\dots,2^{\lfloor \log_2 n \rfloor} - 1\}$ be the number with
binary representation $v$, then we can take $c = c_{k+1}$.
\end{proof}

\begin{lemma}[Small cover]
Let $n \ge 2$ and $\epsilon \in (0,\frac{1}{2})$. Any distribution in $\P_{n,\epsilon}$
has $2\epsilon$-cover of size $2$.
\end{lemma}

\begin{proof}
Consider a distribution $P_i \in \P_{n,\epsilon}$ for some $i \in \{1,2,\dots,n\}$.
Let $j$ be an arbitrary index in $\{1,2,\dots,n\} \setminus \{i\}$.
Consider the projections $c_i, c_j \in C_n$. We claim that $C' = \{c_i, c_j\}$
is a $2\epsilon$-cover of $C_n$.

To see that $C'$ is a $2\epsilon$-cover of $C_n$, consider any $c_k \in C_n$.
We need to show that $d_{P_i}(c_i, c_k) \le 2\epsilon$ or $d_{P_i}(c_j, c_k)
\le 2\epsilon$. If $k = i$ or $k = j$, the condition is trivially satisfied.
Consider $k \in \{1,2,\dots,n\} \setminus \{i,j\}$. Let $X \sim P_i$. Then,
\begin{align*}
d_{P_i}(c_j, c_k)
& = \Pr[c_j(X) \neq c_k(X)] \\
& = \Pr[c_j(X) = 1 \wedge c_k(X) = 0] + \Pr[c_j(X) = 0 \wedge c_k(X) = 1] \\
& = \Pr[X[j] = 1 \wedge X[k] = 0]   + \Pr[X[j] = 0 \wedge X[k] = 1] \\
& = \Pr[X[j] = 1] \Pr[X[k] = 0]  + \Pr[X[j] = 0] \Pr[X[k] = 1] \\
& = 2 \epsilon \left( 1 - \epsilon \right)  \\
& < 2 \epsilon \; .
\end{align*}
\end{proof}

Using Benedek-Itai bound (Theorem~\ref{theorem:benedek-itai} in
\autoref{section:fixed-distribution-learning}) we obtain the corollary below.
The corollary states that the distribution-dependent sample complexity
of learning target in $C_n$ under any distribution from $P_{n,\epsilon}$
does \emph{not} depend on $n$.

\begin{corollary}[Learning with knowledge of the distribution]
Let $n \ge 2$ and $\epsilon \in (0,\frac{1}{2})$.  There exists a
distribution-dependent algorithm such that for any distribution from $\P_{n,\epsilon}$,
any $\delta \in (0,1)$, any target function $c \in C_n$, if the algorithm gets
$$
m \ge \frac{12\ln(2/\delta)}{\epsilon}
$$
labeled examples, with
probability at least $1 - \delta$, it $4\epsilon$-learns the target.
\end{corollary}

The next theorem states that without knowing the distribution,
learning a target under a distribution from $\P_{n,\epsilon}$
requires at least $\Omega(\log n)$ labeled examples. It is important to note that $\epsilon$ in this bound is the parameter of the distribution, and not the accuracy of the PAC learning model.

\begin{theorem}[Learning without knowledge of the distribution]
\label{thm: distribution independent bound}
For any distribution-independent algorithm, any $\epsilon \in (0,\frac{1}{4})$ and any
$n \ge 600/\epsilon^3$ there exists a distribution $P \in \P_{n,\epsilon}$ and a target
concept $c \in C_n$ such that if the algorithm gets
$$
m \le \frac{\ln n}{3 \ln (1/\epsilon)}
$$
labeled examples, it fails to $\frac{1}{16}$-learn the target concept with probability
more than $\frac{1}{16}$.
\end{theorem}

The main idea of the proof is the following. Assume that the learner is restricted to output some function that belongs to $C_n$ (i.e., the learner is \emph{proper}). Then with high probability, the number of coordinates that coincide with the target on a random sample is $\Omega(\epsilon n)$, and, thus, the number of projections that output the same value on each of the $m$ random samples is $\Omega(\epsilon^m n)$. Therefore, with high probability, at least one other projection produces the exact same output as the target. In this case, the learner has to choose randomly, and the probability of choosing a wrong answer is at least $1/2$. This implies that the learner must see at least $m \ge \Omega(\frac{\ln n}{\ln (1/\epsilon)})$ samples. In the proof below we make this intuition formal, and generalize it to the case of improper learners, too.

\begin{proof}[Proof of Theorem~\ref{thm: distribution independent bound}]
Let $A$ be any learning algorithm. For ease of notation, we formalize it is a function
$$
A:\bigcup_{m=0}^\infty \left(\{0,1\}^{m \times n} \times \{0,1\}^m\right) \to \{0,1\}^{\{0,1\}^n} \; .
$$
The algorithm receives an $m \times n$ matrix and a binary vector of length $m$.
The rows of the matrix corresponds to unlabeled examples and the vector encodes
the labels. The output of $A$ is any function from $\{0,1\}^n \to \{0,1\}$.

We demonstrate the existence of a pair $(P,c) \in \P_{n,\epsilon} \times C_n$ which
cannot be learned with $m$ samples by the probabilistic method. Let $I$ be chosen
uniformly at random from $\{1,2,\dots,n\}$. We consider the distribution $P_I \in
\P_{n,\epsilon}$ and target $c_I \in C_n$. Let $X_1, X_2, \dots, X_m$ be an i.i.d.
sample from $P_I$ and let $Y_1 = c_I(X_1), Y_2 = c_I(X_2), \dots, Y_m =
c_I(X_m)$ be the target labels. Let $X$ be the $m \times n$ matrix with entries
$X_i[j]$ and let $Y = (Y_1, Y_2, \dots, Y_m)$ be the vector of labels. The
output of the algorithm is $A(X,Y)$. We will show that
\begin{equation}
\label{equation:projections-failure-probability}
\Exp \left[d_{P_I}(c_I, A(X,Y)) \right] \ge \frac{1}{8} \; .
\end{equation}
This means that there exists $i \in \{1,2,\dots,n\}$ such that
$$
\Exp \left[d_{P_i}(c_i, A(X,Y)) ~\middle|~ I = i \right] \ge \frac{1}{8} \; .
$$
By Proposition~\ref{proposition:error-probability-vs-expected-error},
$$
\Pr \left[ d_{P_i}(c_i, A(X,Y)) > \frac{1}{16} ~\middle|~ I = i \right] \ge \frac{\frac{1}{8} - \frac{1}{16}}{1 - \frac{1}{16}} > \frac{1}{16} \; .
$$

It remains to prove \eqref{equation:projections-failure-probability}. Let $Z$ be
a test sample drawn from $P_I$. That is, conditioned on $I$, the sequence $X_1,
X_2, \dots, X_m, Z$ is i.i.d. drawn from $P_I$. Then, by
Proposition~\ref{proposition:single-bit},
\begin{multline}
\label{equation:projections-expected-error-lower-bound}
\Exp \left[d_{P_I}(c_I, A(X,Y))\right]
= \Pr\left[ A(X,Y)(Z) \neq c_I(Z) \right] \\
\ge
\sum_{\substack{x \in \{0,1\}^{m \times n} \\ y \in \{0,1\}^m \\ z \in \{0,1\}^n}} \left( \frac{1}{2} - \left| \frac{1}{2} - \Exp\left[ c_I(Z) \, \middle| \, X = x, Y = y, Z = z \right] \right| \right) \cdot \Pr \left[X = x, Y = y, Z = z \right]  \; .
\end{multline}
We need to compute $\Exp\left[ c_I(Z) \, \middle| \, X = x, Y = y, Z = z \right]$.
For that we need some additional notation.
For any matrix $x \in \{0,1\}^{m \times n}$, let $x[1], x[2], \dots, x[n]$ be its columns.
For any matrix $x \in \{0,1\}^{m \times n}$ and vector $y \in \{0,1\}^m$ let
$$
k(x,y) = \{ i \in \{1,2,\dots,n\} ~:~ x[i] = y \}
$$
be the set of indices of columns of $x$ equal to the vector $y$. Also, we define
$\norm{\cdot}$ to be the sum of absolute values of entries of a vector or a
matrix. (Since we use $\norm{\cdot}$ only for binary matrices and binary
vectors, it will be just the number of ones.)

For any $i \in \{1,2,\dots,n\}$,
$$
\Pr \left[I = i, X = x, Y = y \right]
=
\begin{cases}
\frac{1}{n} \left( \frac{1}{2} \right)^m \epsilon^{\norm{x} - \norm{y}} (1 - \epsilon)^{mn - \norm{x} + \norm{y}} & \text{if $i \in k(x,y)$,} \\
0 & \text{if $i \not \in k(x,y)$.} \\
\end{cases}
$$
Therefore, for any $i \in \{1,2,\dots,n\}$,
\begin{align*}
\Pr \left[I = i \, \middle| \, X = x, Y = y \right]
& = \frac{\Pr \left[I = i, X = x, Y = y \right]}{\Pr \left[ X = x, Y = y \right]} \\
& = \frac{\Pr \left[I = i, X = x, Y = y \right]}{\sum_{j \in k(x,y)} \Pr \left[ I = j, X = x, Y = y \right]} \\
&  =
\begin{cases}
\frac{1}{|k(x,y)|} & \text{if $i \in k(x,y)$,} \\
0 & \text{if $i \not \in k(x,y)$.} \\
\end{cases}
\end{align*}

Conditioned on $I$, the variables $Z$ and $(X,Y)$ are independent. Thus,
for any $x \in \{0,1\}^n$, and $i = 1,2,\dots,n$,
\begin{align*}
\Pr \left[Z = z \, \middle| \, I = i, X = x, Y = y \right]
& = \Pr \left[Z = z \, \middle| \, I = i \right] \\
& =
\begin{cases}
\frac{1}{2} \epsilon^{\norm{z} - 1} (1 - \epsilon)^{n - \norm{z}} & \text{if $z[i] = 1$,} \\
\frac{1}{2} \epsilon^{\norm{z}} (1 - \epsilon)^{n-1 - \norm{z}}& \text{if $z[i] = 0$.} \\
\end{cases}
\end{align*}
This allows us to compute the conditional probability
\begin{align*}
&\Pr \left[I = i, Z = z \, \middle| \, X = x, Y = y \right] \\
& \quad = \Pr \left[Z = z \, \middle| \, I = i, X = x, Y = y \right] \cdot \Pr \left[I = i \, \middle| \, X = x, Y = y \right] \\
& \quad =
\begin{cases}
\frac{1}{2|k(x,y)|} \epsilon^{\norm{z} - 1} (1 - \epsilon)^{n - \norm{z}} & \text{if $i \in k(x,y)$ and $z[i] = 1$,} \\
\frac{1}{2|k(x,y)|} \epsilon^{\norm{z}} (1 - \epsilon)^{n - 1 - \norm{z}} & \text{if $i \in k(x,y)$ and $z[i] = 0$,} \\
0 & \text{if $i \not \in k(x,y)$.} \\
\end{cases}
\end{align*}
For any $z \in \{0,1\}^n$, let
$$
s(x,y,z) = \{ i \in k(x,y) ~:~ z[i] = 1 \} \; ,
$$
and note that $s(x,y,z) \subseteq k(x,y)$.
Then,
\begin{align*}
& \Pr \left[Z = z \, \middle| \, X = x, Y = y \right] \\
& \quad = \sum_{i=1}^n \Pr \left[Z = z, I = i \, \middle| \, X = x, Y = y \right] \\
& \quad = \sum_{i \in k(x,y)} \Pr \left[Z = z, I = i \, \middle| \, X = x, Y = y \right] \\
& \quad = \sum_{i \in s(x,y,x)} \Pr \left[Z = z, I = i \, \middle| \, X = x, Y = y \right] + \sum_{i \in k(x,y) \setminus s(x,y,z)} \Pr \left[Z = z, I = i \, \middle| \, X = x, Y = y \right] \\
& \quad = \frac{1}{2|k(x,y)|} \cdot |s(x,y,z)| \cdot \epsilon^{\norm{z} - 1} (1 - \epsilon)^{n - \norm{z}} + \frac{1}{2|k(x,y)|} \cdot (|k(x,y)| - |s(x,y,z)|) \cdot \epsilon^{\norm{z}} (1 - \epsilon)^{n - 1 - \norm{z}} \\
& \quad = \frac{\epsilon^{\norm{z} - 1} (1 - \epsilon)^{n - 1 - \norm{z}}}{2|k(x,y)|} \cdot \left( |s(x,y,z)| \cdot (1 - 2\epsilon) + |k(x,y)| \cdot \epsilon \right) \; .
\end{align*}
Hence,
\begin{align*}
& \Exp\left[ c_I(Z) \, \middle| \, X = x, Y = y, Z = z \right] \\
& \quad = \Pr \left[ Z[I] = 1 \, \middle| \, X = x, Y = y, Z = z \right] \\
& \quad = \frac{\displaystyle \Pr \left[ Z[I] = 1, Z = z \, \middle| \, X = x, Y = y \right]}{\displaystyle \Pr \left[ Z = z \, \middle| \, X = x, Y = y \right]} \\
& \quad = \frac{\displaystyle \sum_{i=1}^n \Pr \left[ I = i, Z[i] = 1, Z = z \, \middle| \, X = x, Y = y \right]}{\displaystyle \Pr \left[ Z = z \, \middle| \, X = x, Y = y \right]} \\
& \quad = \frac{\displaystyle \frac{|s(x,y,z)|}{2|k(x,y)|} \cdot \epsilon^{\norm{z} - 1} (1 - \epsilon)^{n - \norm{z}}}{\displaystyle \frac{\epsilon^{\norm{z} - 1} (1 - \epsilon)^{n - 1 - \norm{z}}}{2|k(x,y)|} \cdot \left( |s(x,y,z)| \cdot (1 - 2\epsilon) + |k(x,y,z)| \cdot \epsilon \right)} \\
& \quad = \frac{\displaystyle |s(x,y,z)| \cdot (1 - \epsilon)}{\displaystyle |s(x,y,z)| \cdot (1 - 2\epsilon) + |k(x,y)| \cdot \epsilon } \\
& \quad = \frac{\displaystyle 1 - \epsilon}{\displaystyle 1 - 2\epsilon + \frac{|k(x,y)| \cdot \epsilon}{|s(x,y,z)|}} \\
\end{align*}

We now show that the last expression is close to $1/2$. It is easy to check that
$$
\frac{|k(x,y)| \cdot \epsilon}{|s(x,y,z)|} \in \left[\frac{5}{6}, 2 \right] \qquad \Longrightarrow \qquad \frac{\displaystyle 1 - \epsilon}{\displaystyle 1 - 2\epsilon + \frac{|k(x,y)| \cdot \epsilon}{|s(x,y,z)|}} \in \left[ \frac{1}{4}, \frac{3}{4} \right].
$$
Indeed, since $\epsilon \in (0,\frac{1}{4})$,
$$
\frac{\displaystyle 1 - \epsilon}{\displaystyle 1 - 2\epsilon + \frac{|k(x,y)| \cdot \epsilon}{|s(x,y,z)|}} \ge
\frac{\displaystyle 1 - \epsilon}{\displaystyle 1 - 2\epsilon + 2} \ge \frac{\displaystyle 1 - 1/4}{\displaystyle 1 + 2} = \frac{1}{4}
$$
and
\begin{align*}
\frac{\displaystyle 1 - \epsilon}{\displaystyle 1 - 2\epsilon + \frac{|k(x,y)| \cdot \epsilon}{|s(x,y,z)|}}
\le \frac{\displaystyle 1 - \epsilon}{\displaystyle 1 - 2\epsilon + 5/6}
\le \frac{\displaystyle 1}{\displaystyle 1 - 1/2 + 5/6} = \frac{3}{4} \; .
\end{align*}
We now substitute this into the \eqref{equation:projections-expected-error-lower-bound}. We have
\begin{align*}
& \sum_{\substack{x \in \{0,1\}^{m \times n} \\ y \in \{0,1\}^m \\ z \in \{0,1\}^n}} \left( \frac{1}{2} - \left| \frac{1}{2} - \Exp\left[ c_I(Z) \, \middle| \, X = x, Y = y, Z = z \right] \right| \right)   \cdot \Pr \left[X = x, Y = y, Z = z \right] \\
& = \sum_{\substack{x \in \{0,1\}^{m \times n} \\ y \in \{0,1\}^m \\ z \in \{0,1\}^n}} \left( \frac{1}{2} - \left| \frac{1}{2} - \frac{\displaystyle 1 - \epsilon}{\displaystyle 1 - 2\epsilon + \frac{|k(x,y)| \cdot \epsilon}{|s(x,y,z)|}} \right| \right)  \cdot \Pr \left[X = x, Y = y, Z = z \right] \\
& \ge
\sum_{\substack{x \in \{0,1\}^{m \times n} \\ y \in \{0,1\}^m \\ z \in \{0,1\}^n \\ \frac{|k(x,y,z)| \epsilon}{|s(x,y,z)|} \in [\frac{5}{6},2]}} \left( \frac{1}{2} - \left| \frac{1}{2} - \frac{\displaystyle 1 - \epsilon}{\displaystyle 1 - 2\epsilon + \frac{|k(x,y)| \cdot \epsilon}{|s(x,y,z)|}} \right|  \right)   \cdot \Pr \left[X = x, Y = y, Z = z \right] \\
& \ge
\sum_{\substack{x \in \{0,1\}^{m \times n} \\ y \in \{0,1\}^m \\ z \in \{0,1\}^n \\ \frac{|k(x,y,z)| \epsilon}{|s(x,y,z)|} \in [\frac{5}{6},2]}} \left( \frac{1}{2} - \frac{1}{4} \right) \cdot \Pr \left[X = x, Y = y, Z = z \right] \\
& =
\frac{1}{4} \Pr \left[ \frac{|k(X,Y)| \cdot \epsilon}{|s(X,Y,Z)|} \in \left[\frac{5}{6}, 2 \right]  \right] \; .
\end{align*}
In order to prove \eqref{equation:projections-failure-probability}, we need to show that
$\frac{|k(X,Y)| \cdot \epsilon}{|s(X,Y,Z)|} \in \left[\frac{5}{6}, 2 \right]$ with
probability at least $1/2$. To that end, we define two additional random
variables
$$
K = |k(X,Y)| \qquad \text{and} \qquad S = |s(X,Y,Z)| \; .
$$
The condition $\frac{|k(X,Y)| \cdot \epsilon}{|s(X,Y,Z)|} \in \left[\frac{5}{6}, 2 \right]$ is equivalent to
\begin{equation}
\label{equation:ratio-condition}
\frac{1}{2} \epsilon \le \frac{S}{K} \le \frac{6}{5} \epsilon \; .
\end{equation}

First, we lower bound $K$. For any $y \in \{0,1\}^m$ and any $i,j \in \{1,2,\dots,n\}$,
\begin{align*}
\Pr \left[ j \in k(X,Y)  \, \middle| \, Y = y, I = i \right]  =
\begin{cases}
1 & \text{if $j = i$,} \\
\epsilon^{\norm{y}} (1 - \epsilon)^{m - \norm{y}} & \text{if $j \neq j$.}
\end{cases}
\end{align*}
Conditioned on $Y = y$ and $I=i$, the random variable $K - 1 = |k(X,Y) \setminus \{I\}|$ is
a sum of $n-1$ Bernoulli variables with parameter $\epsilon^{\norm{y}} (1 - \epsilon)^{m - \norm{y}}$, one for each column except for column $i$.
Hoeffding bound with $t = \epsilon^m/2$ and the loose lower bound $\epsilon^{\norm{y}} (1 - \epsilon)^{m - \norm{y}} \ge \epsilon^m$ gives
\begin{align*}
\Pr \left[ \frac{K - 1}{n - 1} > \frac{\epsilon^m}{2}  \, \middle| \,  Y = y, I = i  \right]
& = \Pr \left[ \frac{K - 1}{n - 1} > \epsilon^m - t  \, \middle| \,  Y = y, I = i  \right] \\
& \ge \Pr \left[ \frac{K - 1}{n - 1} > \epsilon^{\norm{y}} (1 - \epsilon)^{m - \norm{y}} - t  \, \middle| \,  Y = y, I = i  \right] \\
& \ge 1 - e^{-2(n-1) t^2} \; .
\end{align*}
Since $m \le \frac{\ln n}{3 \ln (1/\epsilon)}$, we lower bound $t = \frac{\epsilon^m}{2}$ as
$$
t = \epsilon^m/2 > \frac{1}{2} \epsilon^\frac{\ln n}{3 \ln(1/\epsilon)} = \frac{1}{2\sqrt[3]{n}} \; .
$$
Since the lower bound is uniform for all choices of $y$ and $i$, we can remove
the conditioning and conclude that
$$
\Pr \left[ K > 1 + \frac{(n-1)}{2\sqrt[3]{n}} \right] \ge 1 - \exp \left(- \frac{(n-1)}{2 n^{2/3}} \right) \; .
$$
For $n \ge 25$, we can simplify it further to
$$
\Pr \left[ K \ge \frac{n^{2/3}}{2} \right] \ge \frac{3}{4} \; .
$$

Second, conditioned on $K=r$, the random variable $S$
is a sum of $r-1$ Bernoulli random variables with parameter $\epsilon$
and one Bernoulli random variable with parameter $1/2$. Hoeffding bound for any $t \ge 0$
gives that
\begin{align*}
\Pr \left[ \left| \frac{S}{K} - \frac{\epsilon(K - 1) + 1/2}{K} \right| < t \, \middle| \, K = r \right] \ge 1 - 2 e^{-2 r t^2} \; .
\end{align*}
Thus,
\begin{align*}
& \Pr \left[ \left| \frac{S}{K} - \frac{\epsilon(K - 1) + 1/2}{K} \right| < t \ \text{and} \ K \ge \frac{n^{2/3}}{2} \right] \\
&\quad \ge \sum_{r = \lceil n^{2/3} / 2 \rceil}^n \Pr \left[ \left| \frac{S}{K} - \frac{\epsilon(K - 1) + 1/2}{K} \right| < t  \, \middle| \,  K = r \right] \cdot \Pr[K = r] \\
& \quad \ge \sum_{r = \lceil n^{2/3} / 2 \rceil}^n \left( 1 - 2 e^{-2 r t^2} \right) \cdot \Pr[K = r] \\
& \quad \ge \left( 1 - 2 e^{-n^{2/3}  t^2 / 2} \right) \cdot \Pr \left[ K \ge \frac{n^{2/3}}{2} \right] \; .
\end{align*}
We choose $t = \epsilon/4$. Since $n \ge 600/\epsilon^3$, we have $e^{-n^{2/3}  t^2 / 2} < \frac{1}{8}$ and thus
\begin{align*}
\Pr \left[ \left| \frac{S}{K} - \frac{\epsilon(K - 1) + 1/2}{K} \right| < t \ \text{and} \ K \ge \frac{n^{2/3}}{2} \right]
& \ge \left( 1 - 2 e^{-n^{2/3}  t^2 / 2} \right) \cdot \Pr \left[ K \ge \frac{n^{2/3}}{2} \right] \\
& \ge \frac{3}{4} \left( 1 - 2 e^{-n^{2/3}  t^2 / 2} \right) \\
& > \frac{3}{4} \left( 1 - \frac{1}{4} \right) = \frac{9}{16} > \frac{1}{2} \; .
\end{align*}
We claim that $t = \epsilon/4$,
$\left| \frac{S}{K} - \frac{\epsilon(K - 1) + 1/2}{K} \right| < t$
and $K \ge \frac{n^{2/3}}{2}$ imply \eqref{equation:ratio-condition}.
To see that, note that $\left| \frac{S}{K} - \frac{\epsilon(K - 1) + 1/2}{K} \right| < t$ is equivalent to
$$
\frac{\epsilon(K - 1) + 1/2}{K} - t < \frac{S}{K} < \frac{\epsilon(K - 1) + 1/2}{K} + t
$$
which implies that
$$
p \left(1 - \frac{1}{K} \right) - t < \frac{S}{K} < \epsilon \left(1 - \frac{1}{K} \right) + \frac{1}{2K} + t \; .
$$
Since $K \ge \frac{n^{2/3}}{2}$ and $n \ge 25$ we have $K > 4$, which implies that
$$
\frac{3}{4} \epsilon - t < \frac{S}{K} < \frac{3}{4} \epsilon + \frac{1}{2K} + t \; .
$$
Since $K \ge \frac{n^{2/3}}{2}$ and $n \ge \frac{12}{\epsilon^{3/2}}$ we have $K > \frac{5}{2\epsilon}$, which implies that
$$
\frac{3}{4} \epsilon - t < \frac{S}{K} < \frac{3}{4} \epsilon + \frac{\epsilon}{5} + t \; .
$$
Since $t = \epsilon/4$, the condition \eqref{equation:ratio-condition} follows.
\end{proof}

\section{All functions}
\label{section:all-functions}

Let $\X$ be some finite domain. We say a sample $T = ((x_1,y_1),
\dots,(x_m,y_m)) \in (\X \times \{0,1\})^m$ of size $m$ is
\emph{self-consistent} if for any $i,j \in \{1,2,\dots,m\}$, $x_i = x_j$
implies that $y_i = y_j$. A distribution independent algorithm $A$ is said to be
\emph{consistent} if for any self-consistent sample $T = ((x_1,y_1),
\dots,(x_m,y_m)) \in (\X \times \{0,1\})^m$, $A(T)(x_i) = y_i$ holds for any
$i=1,2,\dots,m$.

In this section we show that for $C_{\text{all}} = \{0,1\}^\X$, any consistent
distribution independent learner is almost as powerful as any distribution
independent learner. Note that, in particular, the ERM algorithm for
$C_{\text{all}}$ is consistent. In other words, for the class $C_{\text{all}}$
unlabeled data do \emph{not} have any information theoretic value.

\begin{theorem}[No Gap]
Let $\X$ be some finite domain, $C_{\text{all}} = \{0,1\}^\X$ and $A$ be any
consistent learning algorithm. Then, for any distribution $P$ over $\X$, any
(possibly distribution dependent) learning algorithm $B$ and any $\epsilon, \delta \in (0,1)$,
$$
m(A,C_{\text{all}},P,2\epsilon,2\delta) \le m(B,C_{\text{all}},P,\epsilon,\delta) \; .
$$
\end{theorem}

\begin{proof}
Fix any integer $m \ge 0$ and any distribution $P$ over $\X$. Let $X, X_1, X_2,
\dots, X_m$ be an i.i.d. sample from $P$.
Define the random variable
$$
Z = \Pr[X \not \in \{X_1, X_2, \dots, X_m\} ~|~ X_1, X_2, \dots, X_m] \; .
$$
In other words, $Z$ is the probability mass not covered by $X_1, X_2, \dots, X_m$.
For any $c \in C_{\text{all}}$, let
$T_c = ((X_1, c(X_1)), (X_2, c(X_2)), \dots, (X_m, c(X_m)))$ be the sample labeled
according to $c$. Since $A$ is consistent, with probability one,
for any $c \in C_{\text{all}}$,
\begin{equation}
\label{equation:relate-dP-to-Z}
d_P(A(T_c),c) \leq Z \;.
\end{equation}
Let $\widetilde c$ be chosen uniformly at random from
$C_{\text{all}}$, independently of $X, X_1, X_2, \dots, X_m$.
Additionally, define $\widehat c \in C_{\text{all}}$ as
$$
\widehat c(x) =
\begin{cases}
\widetilde{c}(x) & \text{if $x \in \{X_1, X_2, \dots, X_m\}$,} \\
1 - \widetilde{c}(x) & \text{otherwise}.
\end{cases}
$$
and note that $\widehat c$ and $\widetilde c$ are distributed identically and $T_{\widetilde{c}} = T_{\widehat{c}}$, and  thus
\begin{align}
\Exp \left[ \indicator{d_P \left(B\left( T_{\widetilde c} \right), \widetilde c \right) \ge \epsilon} ~\middle|~ T_{\widetilde{c}} \right]
=\Exp \left[ \indicator{d_P \left(B\left( T_{\widehat c} \right), \widehat c \right) \ge \epsilon} ~\middle|~ T_{\widetilde{c}} \right]
\label{eq: widehatc and widetildec are identically distributed}
\end{align}
We have
\begin{align}
\sup_{c \in C_{\text{all}}} \Pr[d_P(B(T_c),c) \ge \epsilon]
= & \sup_{c \in C_{\text{all}}} \Exp \left[ \indicator{d_P \left(B\left( T_c\right), c \right) \ge \epsilon} \right] \notag \\
\ge & \Exp \left[ \indicator{d_P \left(B\left( T_{\widetilde c} \right), \widetilde c \right) \ge \epsilon} \right] \notag \\
= & \Exp \left[ \Exp \left[ \indicator{d_P \left(B\left( T_{\widetilde c} \right), \widetilde c \right) \ge \epsilon} ~\middle|~ T_{\widetilde{c}} \right] \right] \notag \\
= & \Exp \bigg[\Exp \bigg[\frac{1}{2}  \indicator{d_P \left(B\left( T_{\widetilde c} \right), \widetilde c \right) \ge \epsilon}  + \frac{1}{2} \indicator{d_P \left(B\left( T_{\widehat c} \right), \widehat c \right) \ge \epsilon} ~\bigg|~ T_{\widetilde{c}} \bigg] \bigg]
  \label{eq: applying that widehatc and widetildec are identically distributed}\\
\ge & \Exp \left[ \Exp \left[ \frac{1}{2} \indicator{Z \ge 2 \epsilon} ~\middle|~ T_{\widetilde{c}} \right] \right] \label{equation:switch-to-Z} \\
= & \frac{1}{2} \Exp \left[ \indicator{Z \ge 2 \epsilon} \right] \notag \\
= & \frac{1}{2} \Pr \left[ Z \ge 2 \epsilon \right] \notag \\
= & \frac{1}{2} \sup_{c \in C} \Pr \left[ Z \ge 2 \epsilon \right] \notag \\
\ge & \frac{1}{2} \sup_{c \in C} \Pr \left[ d_P(A(T_c), c) \ge 2 \epsilon \right] \label{equation:use-dP-to-Z} \; .
\end{align}
Equation \eqref{eq: applying that widehatc and widetildec are identically distributed} follows from \eqref{eq: widehatc and widetildec are identically distributed}.
To justify inequality \eqref{equation:switch-to-Z}, note that since
the classifiers $\widetilde{c}$ and $\widehat{c}$ disagree on the missing mass,
if $Z \ge 2\epsilon$ then $d_P(B(T_{\widetilde{c}}), \widetilde{c})
\ge \epsilon$ or $d_P(B(T_{\widehat c}), \widehat c) \ge \epsilon$ or both.
By symmetry between $\widetilde{c}$ and $\widehat{c}$, if $Z \ge 2\epsilon$
then with probability at least $1/2$, $d_P(B(T_{\widetilde{c}}), \widetilde{c})
\ge \epsilon$. Inequality \eqref{equation:use-dP-to-Z} follows from
\eqref{equation:relate-dP-to-Z}.

Since the inequality
$$
\sup_{c \in C_{\text{all}}} \Pr[d_P(B(T_c),c) \ge \epsilon] \ge \frac{1}{2} \sup_{c \in C} \Pr \left[ d_P(A(T_c), c) \ge 2 \epsilon \right]
$$
holds for arbitrary $m$, it
implies $m(A,C_{\text{all}},P,2\epsilon,2\delta) \le m(B,C_{\text{all}},P,\epsilon,\delta)$ for any $\epsilon, \delta \in (0,1)$.
\end{proof}

\section{Conclusion and open problems}
\label{section:conclusions}

\citet{Darnstadt-Simon-Szorenyi-2013} showed that the gap between the number of
samples needed to learn a class of functions of Vapnik-Chervonenkis dimension
$d$ \emph{with} and \emph{without} knowledge of the distribution is
upper-bounded by $O(d)$. We show that this bound is tight for the class of
Boolean projections. On the other hand, for the class of all functions, this gap
is only constant. These observations lead to the following research directions.

First, it will be interesting to understand the value of the gap for larger
classes of functions. For example, one might consider the classes of (monotone)
disjunctions over $\{0,1\}^n$, (monotone) conjuctions over $\{0,1\}^n$, parities
over $\{0,1\}^n$, and halfspaces over $\R^n$. The Vapnik-Chervonenkis dimension
of these classes is $\Theta(n)$ thus the gap for these classes is at least
$\Omega(1)$ and at most $O(n)$. Other than these crude bounds, the question of
what is the gap for these classes is wide open.

Second, as the example with class of all functions shows, the gap is \emph{not}
characterized by the Vapnik-Chervonenkis dimension. It will be interesting to
study other parameters which determine this gap. In particular, it will be
interesting to obtain upper bounds on the gap in terms of other quantities.

Finally, we believe that studying this question in the agnostic extension of the
PAC model \citep[Chapter~2]{Anthony-Bartlett-1999} will be of great interest,
too.

\section*{Acknowledgements}
We thank the anonymous reviewers for their valuable comments.

\bibliography{biblio}
\bibliographystyle{plainnat}

\appendix
\section{Size of $\epsilon$-cover}
\label{section:epsilon-cover}

In this section, we present $(4e/\epsilon)^{d/(1 - 1/e)}$ upper bound on the
size of the $\epsilon$-cover of any concept class of Vapnik-Chervonenkis
dimension $d$. To prove our result, we need Sauer's lemma. Its proof can be
found, for example, in~\citet[Chapter 3]{Anthony-Bartlett-1999}.

\begin{lemma}[Sauer's lemma]
Let $\X$ be a non-empty domain and let $C \subseteq \{0,1\}^\X$ be a concept class
with Vapnik-Chervonenkis dimension $d$. Then, for any $S \subseteq \X$,
$$
\left| \left\{ \pi(c, S) ~:~ c \in C \right\} \right| \le \sum_{i=0}^d \binom{|S|}{i} \; .
$$
\end{lemma}

We remark that if $n \ge d \ge 1$ then
\begin{equation}
\label{equation:sauer-lemma-estimate}
\sum_{i=0}^d \binom{n}{i} \le \left( \frac{ne}{d} \right)^d
\end{equation}
where $e = 2.71828 \dots$ is the base of the natural logarithm. This follows
from the following calculation
\begin{align*}
\left( \frac{d}{n} \right)^d \cdot \sum_{i=0}^d \binom{n}{i}
& \le \sum_{i=0}^d \binom{n}{i} \left( \frac{d}{n} \right)^i \\
& \le \sum_{i=0}^n \binom{n}{i} \left( \frac{d}{n} \right)^i \\
& = \left(1 + \frac{d}{n} \right)^n \le e^d
\end{align*}
where we used in the last step that $1 + x \le e^x$ for any $x \in \R$.

\begin{theorem}[Size of $\epsilon$-cover]
Let $\X$ be a non-empty domain and let $C \subseteq \{0,1\}^\X$ be a concept
class with Vapnik-Chervonenkis dimension $d$. Let $P$ be any distribution over
$\X$. For any $\epsilon \in (0,1]$, there exists a set $C' \subseteq C$ such that
\begin{equation}
\label{equation:theorem-epsilon-cover}
|C'| \le \left( \frac{4e}{\epsilon} \right)^{d/(1-1/e)}
\end{equation}
and for any $c \in C$ there exists $c' \in C'$ such that $d_P(c,c') \le \epsilon$.
\end{theorem}

\begin{proof}
We say that a set $B \subseteq C$ is an \emph{$\epsilon$-packing} if
$$
\forall c,c' \in B \qquad \qquad c \neq c' \quad \Longrightarrow \quad d_P(c,c') > \epsilon
$$
We claim that there exists a maximal $\epsilon$-packing. In order to show that a
maximal set exists we to appeal to Zorn's lemma. Consider the collection of all
$\epsilon$-packings. We impose partial order on them by set inclusion. Notice
that any totally ordered collection $\{ B_i ~:~ i \in I \}$ of
$\epsilon$-packings has an upper bound $\bigcup_{i \in I} B_i$ that is an
$\epsilon$-packing. Indeed, if $c,c' \in \bigcup_{i \in I} B_i$ such that $c
\neq c'$ then there exists $i \in I$ such that $c,c' \in B_i$ since $\{ B_i ~:~
i \in I \}$ is totally ordered. Since $B_i$ is an $\epsilon$-packing, $d_P(c,c') >
\epsilon$. We conclude that $\bigcup_{i \in I} B_i$ is an $\epsilon$-packing. By
Zorn's lemma, there exists a maximal $\epsilon$-packing.

Let $C'$ be a maximal $\epsilon$-packing. We claim that $C'$ is also an
$\epsilon$-cover of $C$. Indeed, for any $c \in C$ there exists $c' \in C'$ such
that $d_P(c,c') \le \epsilon$ since otherwise $C' \cup \{c\}$ would be an
$\epsilon$-packing, which would contradict maximality of $C'$.

It remains to upper bound $|C'|$. Consider any finite subset $C'' \subseteq C'$.
It suffices to show an upper bound on $|C''|$ and since $C''$ is arbitrary, the
same upper bound holds for $|C'|$. Let $M = |C''|$ and let $c_1, c_2, \dots,
c_M$ be concepts in $C''$. For any $i,j \in \{1,2,\dots,M\}$, $i < j$, let
$$
A_{i,j} = \{ x \in \X ~:~ c_i(x) \neq c_j(x) \} \; .
$$
Let $X_1, X_2, \dots, X_K$ be an i.i.d. sample from $P$. We will choose $K$ later.
Since $d_P(c_i, c_j) > \epsilon$,
$$
\Pr[X_k \in A_{i,j}] > \epsilon \qquad \qquad \text{for $k=1,2,\dots,K$}.
$$
Since there are $\binom{M}{2}$ subsets $A_{i,j}$, we have
\begin{align*}
& \Pr\left[\forall i,j, i < j, \ \exists k, \ X_k \in A_{i,j} \right] \\
& \qquad = 1 - \Pr\left[\exists i,j, i < j, \ \forall k, \ X_k \not \in A_{i,j} \right] \\
& \qquad \ge 1 - \sum_{1 \le i < j \le M} \Pr\left[\forall k, \ X_k \not \in A_{i,j} \right] \\
& \qquad \ge 1 - \sum_{1 \le i < j \le M} (1 - \epsilon)^K \\
& \qquad = 1 - \binom{M}{2} (1 - \epsilon)^K \; .
\end{align*}
For $K = \left\lceil \frac{\ln \binom{M}{2}}{\epsilon} \right\rceil +
1$, the above probability is strictly positive. This means there exists a set $S =
\{x_1, x_2, \dots, x_K\} \subseteq X$ such that $A_{i,j} \cap S$ is non-empty
for every $i < j$. This means that for every for every $i \neq j$, $c_i(S) \neq
c_j(S)$ and hence $M = |C''| = \left| \left\{ \pi(c, S) ~:~ c \in C \right\} \right|$.
Thus by Sauer's lemma
$$
M \le \sum_{i=0}^d \binom{K}{i} \; .
$$
We now show that this inequality implies that $M \le (4e/\epsilon)^{d/(1-1/e)}$. We consider several cases.

Case 1: $d = -\infty$. That is, no set is shattered, and $C = \emptyset$.
Then, $M = 0$ and inequality trivially follows.

Case 2: $d = 0$. Then, $M \le 1$ and the inequality trivially follows.

Case 3a: $d \ge 1$ and $M \le e^d$. Clearly, $M \le e^d \le (4e/\epsilon)^{d/(1 - 1/e)}$.

Case 3b: $d \ge 1$ and $M > e^d$. Then, $K \ge \ln M \ge d$ and
hence by \eqref{equation:sauer-lemma-estimate},
\begin{align*}
M
\le \sum_{i=0}^d \binom{K}{i}
\le \left( \frac{Ke}{d} \right)^d \; .
\end{align*}
Thus,
\begin{align*}
\ln M
& \le d \ln \left( \frac{Ke}{d} \right) \\
& \le d \ln \left( \frac{e \left(\left\lceil \frac{\ln \binom{M}{2}}{\epsilon} \right\rceil + 1 \right)}{d} \right) \\
& \le d \ln \left( \frac{e \left(\frac{\ln \binom{M}{2}}{\epsilon} + 2 \right)}{d} \right) \\
& \le d \ln \left( \frac{e \left(\frac{\ln \binom{M}{2} + 2}{\epsilon} \right)}{d} \right) \\
& \le d \ln \left( \frac{e \left(\frac{2\ln M + 2}{\epsilon} \right)}{d} \right) \\
& \le d \ln \left( \frac{e \left(\frac{4\ln M}{\epsilon} \right)}{d} \right) \\
& = d \left[ \ln \left( \frac{4e}{\epsilon} \right)  + \ln \left(\frac{\ln M}{d} \right) \right] \\
& \le d \ln \left( \frac{4e}{\epsilon} \right) + \frac{1}{e} \ln M \; .
\end{align*}
where in the last step we used that $\ln x \le x/e$ for any $x > 0$.
Hence,
$$
(1 - 1/e) \ln M \le d \ln \left( \frac{4e}{\epsilon} \right)
$$
which implies the lemma.
\end{proof}

\section{Fixed distribution learning}
\label{section:fixed-distribution-learning}

\begin{theorem}[Chernoff--Hoeffding bound,~\cite{h63}]
Let $X_1, X_2, \dots, X_n$ be i.i.d. Bernoulli random variables with $\Exp[X_i] = p$.
Then, for any $\epsilon \in [0, \min\{p,1-p\})$,
\begin{align*}
\Pr \left[{\frac {1}{n}} \sum_{i=1}^n X_i \ge p + \epsilon \right] \le e^{ - n \KL{p + \epsilon}{p}}  \; , \\
\Pr \left[{\frac {1}{n}} \sum_{i=1}^n X_i \le p - \epsilon \right] \le e^{ - n \KL{p - \epsilon}{p}}  \; .
\end{align*}
where
$$
\KL{x}{y} = x \ln \left( \frac{x}{y} \right) + (1 - x) \ln \left( \frac{1-x}{1-y} \right)
$$
is the Kullback-Leibler divergence between Bernoulli distributions with parameters $x, y \in [0,1]$.
\end{theorem}

We further use the following inequality
$$
\KL{x}{y} \ge \frac{(x-y)^2}{2 \max\{x, y\}}
$$

\begin{theorem}[Benedek-Itai]
\label{theorem:benedek-itai}
Let $C \subseteq \{0,1\}^\X$ be a concept class over a non-empty domain $\X$.
Let $P$ be a distribution over $\X$. Let $\epsilon \in (0,1]$ and assume that
$C$ has an $\frac{\epsilon}{2}$-cover of size at most $N$. Then, there exists an
algorithm, such that for any $\delta \in (0,1)$, any target $c \in C$, if it
gets
$$
m \ge 48\left(\frac{\ln N + \ln(1/\delta)}{\epsilon}\right)
$$
labeled samples then with probability at least $1 - \delta$, it
$\epsilon$-learns the target.
\end{theorem}

\begin{proof}
Given a labeled sample $T = ( (x_1, y_1), \dots, (x_m, y_m) )$, for any $c \in C$,
we define
$$
\err_T(c) = \frac{1}{m} \sum_{i=1}^m \indicator{c(x_i) \neq y_i} \; .
$$

Let $C' \subseteq C$ be an $(\epsilon/2)$-cover of size at most $N$.
Consider the algorithm $A$ that given a labeled sample $T$ outputs
$$
\widehat c = \argmin_{c' \in C'} \err_T(c')
$$
breaking ties arbitrarily. We prove that $A$, with probability at least
$1-\delta$, $\epsilon$-learns any target $c \in C$ under the distribution $P$.

Consider any target $c \in C$. Then there exists $\widetilde c \in C'$ such that
$d_P(c,\widetilde c) \le \epsilon/2$. Let $C'' = \{ c' ~:~ d_P(c,c') > \epsilon \}$.
We claim that with probability at least $1 - \delta$, for all $c' \in C''$,
$\err_T(c') > \frac{2}{3} \epsilon$ and $\err_T(\widetilde{c}) <
\frac{2}{3}\epsilon$ and hence $A$ outputs $\widehat c \in C' \setminus C''$.

Consider any $c' \in C''$ and note that $\err_T(c')$ is an average of Bernoulli
random variables with mean $d_P(c,c') > \epsilon$. Thus, by Chernoff bound,
\begin{align*}
\Pr \left[ \err_T(c') > \frac{2}{3} \epsilon \right]
& > 1 - \exp \left( - m \KL{\frac{2}{3} \epsilon}{d_P(c,c')} \right) \\
& > 1 - \exp \left( - m \frac{(\frac{2}{3} \epsilon - d_P(c,c'))^2}{2 d_P(c,c')} \right) \\
& > 1 - \exp \left( - m \epsilon/18 \right)
\end{align*}
where the last inequality follows from the inequality
$$
\left( \frac{2}{3} \epsilon - x \right)^2 \ge \frac{1}{9} \epsilon x
$$
valid for any $x \ge \epsilon > 0$.
Similarly, $\err_T(\widetilde{c})$ is an average of Bernoulli random variables with mean $d_P(c, \widetilde{c}) < \epsilon/2$.
Thus, by Chernoff bound,
\begin{align*}
\Pr \left[ \err_T(\widetilde{c}) < \frac{2}{3} \epsilon \right]
& > 1 - \exp \left( - m \KL{\frac{2}{3} \epsilon}{d_P(c, \widetilde{c})} \right) \\
& > 1 - \exp \left( - m \frac{(\frac{2}{3} \epsilon - d_P(c, \widetilde{c}))^2}{\frac{4}{3} \epsilon} \right) \\
& > 1 - \exp \left( - m \epsilon / 48 \right) \; .
\end{align*}
Since $|C''| \le N - 1$, by union bound, with probability at least $1 - (N - 1)
\exp(-m \epsilon/48)$, for all $c' \in C''$, $\err_T(c') > \frac{2}{3}
\epsilon$. Finally, with probability at least $1 - N \exp(-m \epsilon/48) \ge 1 -
\delta$, $\err_T(\widetilde{c}) < \frac{2}{3}\epsilon$ and for all $c' \in C''$,
$\err_T(c') > \frac{2}{3} \epsilon$.
\end{proof}

\end{document}